\documentclass[letterpaper, 10 pt, conference]{ieeeconf}  

\IEEEoverridecommandlockouts                              

\usepackage{graphics} 
\usepackage{epsfig} 
\usepackage{times} 

\usepackage{amsmath} 
\usepackage{amssymb}  
\usepackage{amsthm}
\usepackage{algorithm}
\usepackage{algpseudocode}
\usepackage{subcaption}
\usepackage{booktabs}
\usepackage{graphicx}
\usepackage{arydshln}
\usepackage{xcolor}
\usepackage{hyperref}

\newtheorem{theorem}{Theorem}
\newtheorem{proposition}{Proposition}

\title{\LARGE \bf
When World Models Lie: Adaptive Safety Analysis Under \\ Wrong Imaginations
}

\author{John Cao and Somil Bansal%
\thanks{This research is supported in part by the NSF CAREER program (2240163), the TRI University Research Program (URP) 3.0, and the NASA ULI program.}%
\thanks{The authors are with the Department of Aeronautics
and Astronautics, Stanford University.
{\tt\small \{johncao,somil\}@stanford.edu}}%
}

\begin{document}

\maketitle
\thispagestyle{empty}
\pagestyle{empty}

\begin{abstract}
World models offer a powerful substrate for safety reasoning in high-dimensional robotic systems, but they are also fallible: their predictions can be biased, miscalibrated, or confidently wrong. This creates a central challenge for latent-space safety filters, which often learn Hamilton-Jacobi safety value functions on the dynamics of a world model. If the world model is incorrect, the resulting value function can inherit its errors and produce overconfident safety estimates. 
Existing latent safety filters often rely on auxiliary signals such as ensemble disagreement or value-target consistency residuals for adaptation, but these signals can remain small even when the world model’s predictions deviate from observations. We propose an adaptive latent safety filter that calibrates safety reasoning using directly observed world-model error. Our method uses Adaptive Conformal Inference to construct online uncertainty sets from discrepancies between predicted and observation-inferred latent states, then evaluates safety pessimistically by minimizing the learned value function over these sets. This allows the filter to remain minimally conservative when the world model is accurate, while becoming more cautious when observations reveal model mismatch. We provide a finite-time coverage guarantee for the adaptive uncertainty radius. Through simulation and hardware experiments, we show that our method significantly reduces failures relative to state-of-the-art latent safety filters while preserving task completion. Our project page is available at \url{https://mudhdhoo.github.io/WhenWordModelsLie_project_page/}.
\end{abstract}

\section{Introduction}

Robotic systems are becoming increasingly capable of performing sophisticated tasks in complex, unstructured, and human-centric environments~\cite{kim2025cosmospolicy, intelligence2025pi_, team2025gemini}. As these systems become more capable, it is increasingly important to ensure that they can act without compromising the safety of their surroundings. Safety filters provide a mechanism for doing so by overriding or modifying actions that may lead to unsafe outcomes. Control Barrier Functions~\cite{ames2019control} and Hamilton-Jacobi (HJ) reachability~\cite{bansal2017hamilton} are two dominant frameworks for safety filtering, but their reliance on explicit dynamics models and safety constraints has historically limited their application in open-world settings.
\begin{figure}[!t]
    \centering
    \includegraphics[width=\linewidth]{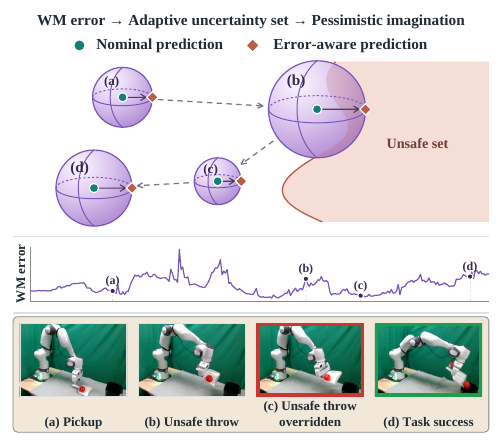}
    \caption{ 
    Our method grounds safety adaptation directly in observed prediction error. We compare the latent state predicted by the world model with the latent state inferred from the current observation, compute an uncertainty set over plausible
latent prediction deviations, and then reason \textit{pessimistically }over the resulting uncertainty set. Pessimism scales with the size of the errors, providing more caution under an inaccurate world model, while being non-invasive when the model is reliable.
\vspace{-0.5cm}
    }
    \label{fig:front_fig}
\end{figure}
Latent-space safety filters have recently emerged as a promising paradigm for safety filtering from high-dimensional observations~\cite{nakamura2025generalizing}. These methods learn a compact latent world model~\cite{hafner2019learning} and train an HJ-style safety value function using rollouts imagined by the learned latent dynamics. The resulting value function can then be used to evaluate the safety of a robot's actions from image-based observations, enabling safety filtering in settings where classical model-based methods are difficult to apply.
However, latent-space safety filtering introduces a fundamental challenge: the model used for safety reasoning is itself learned and fallible. Classical CBF and HJ reachability methods typically assume access to a known dynamics model and well-specified failure set, under which formal safety guarantees can be derived. In contrast, learned world models are only approximate representations of the real world~\cite{abbeel2006using}. They may be naturally erroneous even under nominal deployment conditions, and their errors can grow under distribution shift, unmodeled dynamics, or changes in the environment. 
Indeed, limiting errors and hallucinations in world models is an active area of research in robotics \cite{hansen2026hallucination, li2025robotic, an2026feedback}. This is especially problematic for latent safety filters because errors in the world model can propagate into the learned safety value function. An optimistic world model may imagine benign outcomes for unsafe actions, causing the value function to become overconfident and delaying safety interventions precisely when they are needed most.

Prior work has begun to address the fragility of latent safety filters under model error. World-model-based runtime monitors have been proposed to detect unsafe behaviors~\cite{ward2026foundational}, but detection alone does not resolve how the robot should intervene safely. UNISafe~\cite{seo2025uncertainty} treats out-of-distribution latent states as unsafe by using the Jensen-Rényi Divergence (JRD) of a probabilistic prediction ensemble as an uncertainty score. However, as we demonstrate in our experiments, ensemble uncertainty can remain low even when predictions are systematically biased or incorrect~\cite{berger2026biased}, leading the system to be confidently wrong. ACoFi~\cite{huriot2026safe} adapts the safety threshold using the discrepancy between the learned value function and its bootstrapped training target. Yet when the value function is trained on erroneous latent dynamics, this residual may reflect self-consistency within the incorrect model rather than disagreement with the real system. As a result, auxiliary uncertainty or value-consistency signals may fail to reveal the underlying prediction errors that compromise safety.

In this work, we propose an adaptive latent safety filter grounded in observed discrepancies between predicted and inferred latent states. Using Adaptive Conformal Inference (ACI) \cite{gibbs2021adaptive}, we construct online uncertainty sets and minimize the nominal safety value over these sets, adapting the filter’s conservatism to observed world-model error. Our contributions are as follows:
\begin{enumerate}
\item We propose an adaptive latent safety filter that uses Adaptive Conformal Inference to construct uncertainty sets from observed world-model prediction errors and evaluates safety pessimistically by minimizing the nominal value function over these sets.
\item We provide a finite-time coverage guarantee for the adaptive uncertainty radius used by our method.
\item Through both qualitative and quantitative case studies across simulation and hardware, we show how our approach overcomes the limitations of baselines and its effectiveness in reducing failures despite incorrect world models, while preserving task performance.
\end{enumerate}

\section{Problem Formulation}
We study the problem of online safety reasoning with fallible world models. Consider a latent world model
\begin{align}\label{world_model}
    z_t \sim \mathcal{E}(\cdot | \hat{z}_{t}, o_t) \hspace{0.5cm} \hat{z}_t \sim f(\cdot | z_{t-1}, a_{t-1}),
\end{align}
where $\mathcal{E}$ is an encoder that infers the current latent state from the predicted latent state $\hat{z}_{t}$ and observation $o_t$, and $f$ is a learned latent dynamics model that autoregressively predicts the next latent state. 
For instance, $f$ can be trained offline using a dataset of robot--environment interactions, consisting of observation and action trajectories \cite{ward2026foundational, nakamura2025generalizing}.
In this work, we are particularly concerned with the setting where the world model $f$ might be wrong, e.g., due to learning errors or hallucinations.

The system safety constraints are specified in the latent space via a failure set $\mathcal{F} = \{z : l(z) < 0\}$, encoded via the
zero-sublevel set of a margin function $l$. In practice, $l$ is learned as a binary classifier on known safe/unsafe trajectories \cite{nakamura2025train}.
We also assume access to a nominal safety value function $V(z)$ that encodes the safety of a latent state under the nominal world model and can be used for monitoring, planning, or filtering. 
Such a value function can be trained or computed using the latent dynamics in \eqref{world_model}, using reachability analysis for instance \cite{nakamura2025generalizing} (see Sec. \ref{sec:latent_safety_background}).

However, because the latent dynamics $f$ only approximate the true evolution of the system, the nominal value function $V$ may be miscalibrated during deployment. In particular, prediction errors in the world model can induce errors in the estimated safety value, causing the safety filter to be either overly conservative or, more dangerously, overconfident in unsafe regions. Our goal is to adapt the safety analysis online by using real-time observations to estimate the world model’s prediction error and compensate for its effect on the safety value.
Specifically, this work aims to:
\begin{enumerate}
    \item Estimate the prediction errors made by the world model during deployment.
    \item Use this error estimate to compute an adapted safety value function $\bar{V}$ that accounts for model error and reduces overconfident safety predictions.
\end{enumerate}

\section{Background}

\subsection{Latent Space Safety Filters} \label{sec:latent_safety_background}
Latent space safety filters perform HJ Reachability in the latent space of world models \cite{nakamura2025generalizing}. Given a set of failure states $\mathcal{F}$ and the margin function $l$,
HJ Reachability analysis defines the \textit{value function} under a dynamics model $f$ as the optimal control problem $V(z_t) = \sup_\pi \inf_{s \geq 0} l(\xi_{z_t}^\pi(s))$, where $\xi_{z_t}^\pi$ is a trajectory starting from $z_t$ and evolving under the dynamics of $f$ under the policy $\pi$. This can be interpreted as the worst future cost incurred when using the maximally safe policy; $V <0$ therefore means that the system is doomed to fail despite the best efforts. The value function satisfies the fixed-point Bellman equation
\begin{align}\label{vf_bellman}
    V(z_t) = \min \{l(z_t), \max_a V(z_{t+1})\}.
\end{align}
In latent space safety filtering, $l$ is learned as a classifier from a dataset of high-dimensional failure-labeled observations. The value function $V(z)$ is learned using reinforcement learning (RL) on the world model, and \eqref{vf_bellman} is replaced with the corresponding discounted contraction mapping
\begin{align}
    V(z_t) = (1-\gamma)l(z_t) + \gamma\min \{l(z_t), \max_a V(\hat{z}_{t+1})\}
\end{align}
where $\hat{z}_{t+1} \sim f(\cdot | z_{t}, a_{t})$ and $\gamma \in [0,1)$ is the discount factor \cite{fisac2019bridging}. This safety Bellman equation is used within RL as a backup target to jointly learn $V$ (critic) and a \textit{safety policy} (actor) $\pi_\text{safe}(z) = \arg\max_a V(z)$ that computes the maximally safe action at a given state. Safety filtering for a nominal policy $\pi_\text{nom}$ is then implemented as
\begin{align}\label{safety_filtering}
    a_t \sim \pi = \mathbf{1}\left[ V(\hat{z}_{t+1}) > \delta \right]\pi_\text{nom} +  \mathbf{1}\left[ V(\hat{z}_{t+1}) \leq \delta \right]\pi_\text{safe}
\end{align}
where $\delta \geq 0$ is a user-defined safety threshold.

\subsection{Adaptive Conformal Inference}
Adaptive Conformal Inference (ACI) \cite{gibbs2021adaptive} calibrates prediction sets online under distribution shift. Given an input–output pair $(x_t,y_t)$, let $s_t=S_t(x_t,y_t)$ be a finite nonconformity score, with larger values indicating poorer agreement with the prediction. Using a non-decreasing quantile function $\widehat{Q}_t$ fitted to observed scores, ACI constructs
\begin{align}
 C_t=\{y:S_t(x_t,y)\leq q_t\}, \qquad q_t=\widehat{Q}_t(1-\alpha_t).   
\end{align}
For a miscoverage level $\alpha\in(0,1)$, the adaptive parameter is initialized as $\alpha_1=\alpha$ and updated after observing $y_t$: \begin{align}\label{aci_update}
\alpha_{t+1} = \alpha_t+\lambda(\alpha-\text{err}_t), \qquad \text{err}_t = \mathbf{1}[s_t>q_t],
\end{align}
where $\lambda>0$ is a step size controlling the adaptation rate. Miscoverage decreases $\alpha_t$, inducing larger subsequent prediction sets, while coverage conversely allows for tighter sets. Larger $\lambda$ enables faster adaptation but trades off stability.

\begin{proposition}[ACI coverage~{\cite[Proposition~4.1]{gibbs2021adaptive}}]
\label{prop:aci_coverage}
Suppose $\alpha_1=\alpha\in(0,1)$ and the update
\eqref{aci_update} is applied with $\lambda>0$ without
clipping $\alpha_t$. Under the conventions
$\widehat{Q}_t(p)=-\infty$ for $p\leq0$ and
$\widehat{Q}_t(p)=+\infty$ for $p\geq1$, with probability 1,
for all $T \in \mathbb{N}$,
\begin{align}
    \left|
        \frac{1}{T}\sum_{t=1}^{T}\mathrm{err}_t-\alpha
    \right|
    \leq
    \frac{\max\{\alpha,1-\alpha\}+\lambda}{T\lambda}.
    \label{eq:aci_coverage_bound}
\end{align}
\end{proposition}

Proposition~\ref{prop:aci_coverage} puts no constraints on the distribution generating the data, meaning that \eqref{eq:aci_coverage_bound} is valid even when exchangeability does not hold between subsequent time steps. In our method, latent prediction errors provide the nonconformity scores used to construct uncertainty sets for pessimistic safety evaluation.

\section{Method}
We propose an online safety-filtering framework that enables safety reasoning with fallible latent world models. The key idea is to avoid treating the world model prediction as a point estimate that can be trusted uniformly during deployment. Instead, we use the discrepancy between predicted and inferred latent states as real-time feedback about the reliability of the world model, and use this feedback to construct a pessimistic safety value that accounts for possible model errors.

At each time step, the world model predicts a latent state $\hat{z}_t$ using the learned latent dynamics, while the encoder infers a latent state $z_t$ from the current observation. We define the instantaneous latent prediction error as $e_t = ||z_t - \hat{z}_t||_2$.
This error provides a simple online signal of how well the latent dynamics are tracking the true system evolution. When the prediction error is small, the nominal value function can be trusted more closely. When the prediction error grows, the safety filter should become more cautious, since the nominal value function may be overconfident in regions where the world model is inaccurate.

\begin{algorithm}[t]
\caption{Pessimistic Latent Safety Filtering}
\label{alg:pess_safety_filering}
\begin{algorithmic}[1]
\Require Encoder $\mathcal{E}$, dynamics $f$, value function $V$, nominal policy $\pi_\text{nom}$, safety policy $\pi_\text{safe}$ trust radius $q_\text{max}$, error buffer $\mathcal{B}$, threshold $\delta$

\For{$t = 1:T$}

\State $z_t \sim \mathcal{E}(\cdot | \hat{z}_{t}, o_t)$ \Comment{Observe and encode}

\State $e_t = ||z_t - \hat{z}_t||_2$ \Comment{Compute prediction error}

\State Append $e_t$ to buffer $\mathcal{B}$

\State Estimate uncertainty set $\mathcal{U}_{t+1}(\mathcal{B})$

\State $\hat{z}_{t+1} \sim f(\cdot | z_t, \pi_\text{nom}(z_t))$ \Comment{Estimate next state} 

\State Compute pessimistic value function
\begin{align}\label{opt_prob}
    \bar{V}(\hat{z}_{t+1}) = &\min_z V(z) \\ \notag 
    &\text{ s.t} \hspace{0.5 cm} z \in \mathcal{U}_{t+1} \notag
 \end{align}

\State Filter the nominal action using $\bar{V}$ according to \eqref{safety_filtering}.

\EndFor

\end{algorithmic}
\end{algorithm}

\subsection{Online Safety Adaptation via Pessimistic Imaginations}
We introduce \textit{pessimistic imaginations} for latent safety filtering, outlined in Algorithm \ref{alg:pess_safety_filering}. 
At each time step, we compute the latent prediction error $e_t$ and add it to a buffer of previously observed errors. This buffer is used to construct an uncertainty set $\mathcal{U}_{t+1}$ around the next latent prediction $\hat{z}_{t+1}$. 
The uncertainty set represents latent states that are plausible given the recent mismatch between the world model and the observed system behavior. Rather than evaluating safety only at the predicted latent state $\hat{z}_{t+1}$, we evaluate safety pessimistically over $\mathcal{U}_{t+1}$.
In particular, we construct an adapted safety value
$\bar{V}$ by solving the corresponding constrained optimization problem in \eqref{opt_prob} using projected gradient descent (PGD). This procedure has an intuitive interpretation: The resulting value $\bar{V}$ reflects the worst-case safety outcome among latent states that are consistent with the current estimate of model error. We then use $\bar{V}$ in place of the nominal value $V(\hat{z}_{t+1})$ for filtering the proposed action. As a result, the filter remains minimally conservative when the world model is accurate, but becomes more cautious when online observations reveal that the model is unreliable.

\subsection{Adaptive Conformal Inference for Online Uncertainty Estimation}\label{aci_estimation}
We use ACI \cite{gibbs2021adaptive} to construct the uncertainty set $\mathcal{U}_{t+1}$. While Algorithm \ref{alg:pess_safety_filering} is not tied to a particular uncertainty-estimation method, ACI provides a simple distribution-free mechanism for adapting uncertainty sets under non-stationary deployment conditions. This is especially useful in the context of learned world models, whose prediction errors may vary across tasks, environments, and phases of execution. We define the nonconformity score at time $t$ as the latent prediction error
\begin{align}
    s_t = e_t = ||z_t - \hat{z}_t||_2.
\end{align}
Given a target miscoverage level $\alpha$, the adaptive parameter is updated according to \eqref{aci_update}, with $\text{err}_t = \mathbf{1}\left[||z_t - \hat{z}_t||_2 > q_t\right]$.
We then compute an empirical quantile $q_{t+1}$ of the observed nonconformity scores and define the next uncertainty set as
\begin{align}
    \mathcal{U}_{t+1} = \{z : ||z - \hat{z}_{t+1}||_2 \leq \min(q_{t+1}, q_\text{max})\},
\end{align}
where $q_{t+1}$ is the fitted $\frac{\lceil (t + 1)(1 - \alpha_{t+1})\rceil}{t}$ quantile of the nonconformity scores, and $q_\text{max}$ defines a trust-region radius that limits the pessimistic optimization to a calibrated neighborhood of the predicted latent state.

The trust-region radius $q_\text{max}$ controls the tradeoff  between safety and conservatism. A larger radius allows the filter to guard against a broader range of model errors, but may introduce excessive conservatism by considering latent states that are not physically meaningful.  
A smaller radius keeps the pessimistic imagination closer to the learned latent manifold, but may fail to account for safety-relevant prediction errors. 
To choose $q_\text{max}$ systematically, we leverage trajectory-level conformal prediction \cite{seo2025uncertainty, ren2023robots}. 
We construct a calibration dataset $\mathcal{D}_\text{cal}=  \{\tau_i\}_{i=1}^{N_\text{cal}}$ with trajectories $\tau_i = \{o_t^i,a_t^i\}_{t=1}^{T}$. 
For each trajectory $\tau_i$, we compute the sequence of latent prediction errors ${e_1^i,\ldots,e_{T}^i}$ and define the \textit{trajectory-level} nonconformity score $Q_{\tau_i}^{\alpha_{\mathrm{traj}}}$ as the $(1-\alpha_{\mathrm{traj}})$ quantile of these errors. 
The trust region radius $q_\text{max}$ is then chosen as the $1-\alpha_\text{cal}$ quantile of the computed trajectory-level scores $\{Q_{\tau_i}^{\alpha_\text{traj}}\}_{i=1}^{N_\text{cal}}$.
Under exchangeability between a new trajectory $\tau_{\mathrm{new}}$ and the calibration set $\mathcal{D}_{\mathrm{cal}}$, conformal calibration gives
\begin{align}\label{cp_bound}
    \mathbb{P}(Q_{\tau_\text{new}}^{\alpha_\text{traj}} \leq q_\text{max}) \geq 1-\alpha_\text{cal}.
\end{align}
This guarantee does not imply that every latent state considered by the pessimistic optimization in \eqref{opt_prob} is physically realizable. Rather, it provides a principled way to choose the maximum adaptation radius so that the filter can trade off sensitivity to world-model error against unnecessary conservatism within the calibrated operating distribution.

When little online data has been observed, we initialize $q_1 = \infty$, which makes the uncertainty radius equal to the trust-region radius $q_{\max}$. Thus, the filter begins cautiously by searching over the full calibrated trust region, and then becomes more adaptive as deployment data accumulates.

\subsection{Coverage Guarantee}
The following result combines trajectory-level calibration of $q_{\max}$ with the online ACI update used to construct $\mathcal{U}_{t+1}$.
\begin{theorem}\label{thm_1}
Let $q_\text{max}$ be chosen using the procedure in Section \ref{aci_estimation}, and $\mathcal{U}_{t}$ be computed using Adaptive Conformal Inference with radius $\bar{q}_t = \min(q_t, q_\text{max})$. Let $r_T = \frac{\max(\alpha, 1-\alpha) + \lambda}{T\lambda}$ and $\mathbf{1}$ be the indicator function. Given a new trajectory $\tau$, assuming exchangeability with $\mathcal{D}_\text{cal}$,
\begin{align}\label{thm1_ineq}
    \mathbb{P}\left( \frac{1}{T}\sum_{t=1}^T \mathbf{1}[e_t > \bar{q}_t] \leq \alpha + \alpha_\text{traj} + r_T\right) \geq 1 - \alpha_\text{cal}.
\end{align}
\end{theorem}

The proof for Theorem \ref{thm_1} is provided in the appendix. Intuitively, this theorem states that, with high probability over calibrated trajectories, the fraction of time steps at which the realized latent prediction error exceeds the adaptive uncertainty radius is controlled by the target online miscoverage level, the trajectory-level calibration tolerance, and a finite-time adaptation term. In our setting, this bound provides a statistical basis for using $\mathcal{U}_{t+1}$ as the set of plausible latent deviations over which the safety value is pessimistically evaluated.

\section{Case Studies}
We evaluate our method in two simulated case studies. First, we use Dubins car to isolate the properties and benefits of error-feedback-based safety adaptation. We then perform a quantitative study on a manipulation task using a Franka Panda arm. 
Together, these case studies evaluate whether latent safety filters can remain effective when the world model used for safety reasoning is fallible.

\begin{figure}[t]
    \centering
    \includegraphics[width=\linewidth]{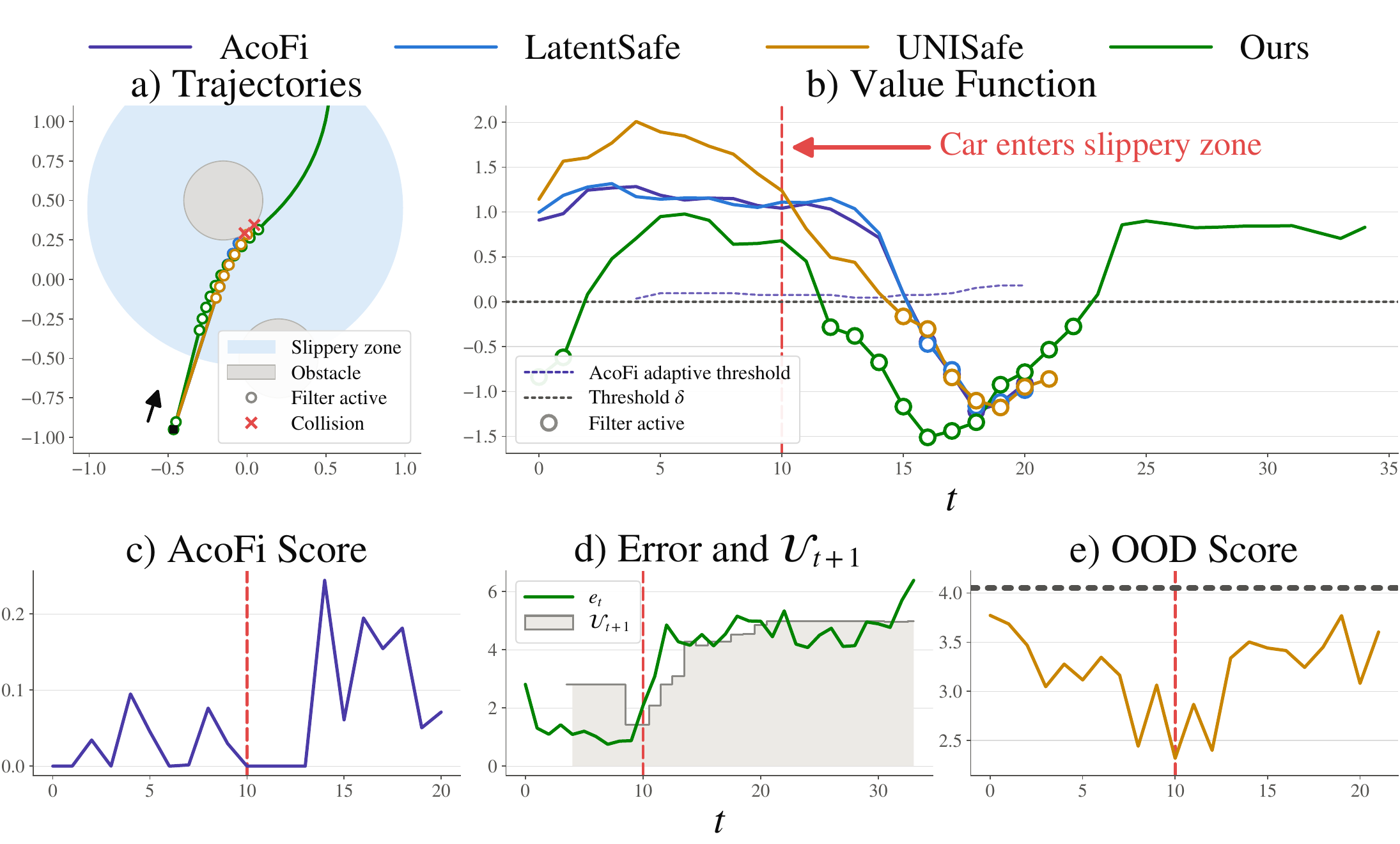}
    \caption{Case Study: Prediction errors allow our method to engage safety fallback earlier to avoid collision. Maximum pessimism is held at the start, when data is scarce.}
    \label{fig:dubins_study}
\end{figure}

\textbf{Baselines.} We compare our method to four baselines: 1) an unfiltered nominal policy, 2) LatentSafe \cite{nakamura2025generalizing}, 3) UNISafe \cite{seo2025uncertainty}, and 4) AcoFi \cite{huriot2026safe}. For all baselines, we train a DreamerV3 \cite{hafner2023mastering} world model for each environment, trained on a dataset of success and failure trajectories. The safety margin function is trained on the same dataset using the method in \cite{nakamura2025train}. The safety filters are then trained using RL within the imagination of the world model, as described in Sec. \ref{sec:latent_safety_background}. 
We aim to demonstrate how our approach addresses blind spots in current methods that can otherwise lead to safety violations when the world model is incorrect.

\textbf{Metrics.} We evaluate methods using three metrics: 1) \textit{Success}: the number of episodes in which the intended task is achieved without safety violations; 2) \textit{Failure}: the number of episodes with safety violations; 3) \textit{Incompletion}: the number of episodes that neither fail nor succeed.

\subsection{Case Study 1: Slippery Dubins Car} 
We evaluate error-feedback-based safety adaptation using a Dubins car with a “slippery” region that increases the car's velocity by a factor of 1.5 relative to the training dynamics. We train a world model on a dataset of 1500 nominal trajectories collected without the slippery zone, each composed of failure-labeled observation-action pairs $\{o_t, a_t, y_t\}_{t=1}^T$, where $o_t$ is an image snapshot of the environment at time $t$, $a_t$ is the corresponding action taken, and $y_t$ is a failure label assigned a value of $-1$ when the car collides with an obstacle and $1$ otherwise. The world model uses the DreamerV3 recurrent state-space model (RSSM) architecture with a CNN image encoder, while $l$ and $V$ are parameterized as feed-forward MLPs.

The dynamics shift causes the car to travel farther per time-step than the world model anticipates. Figure \ref{fig:dubins_study}(a–b) shows that the baselines collide because their safety estimates fail to respond sufficiently to this shift. In contrast, our latent prediction error rises upon entering the slippery region (Figure \ref{fig:dubins_study}(d)), expanding the uncertainty set and lowering the pessimistic safety value. This triggers an early intervention that steers the car away from the obstacle.
\begin{figure}[t]
    \centering
    \includegraphics[width=\linewidth]{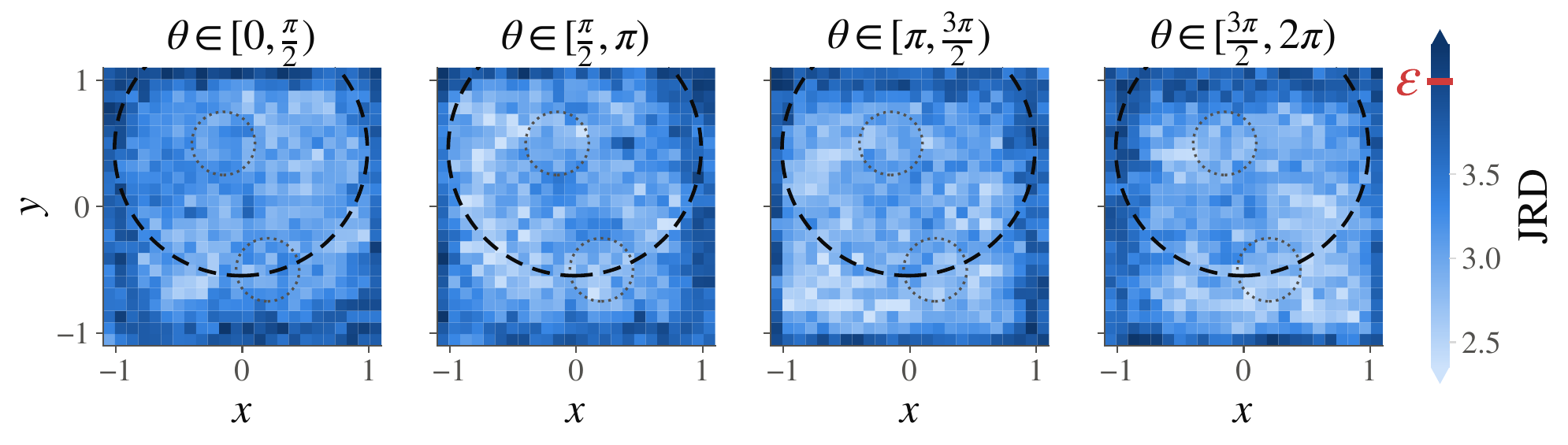}
    \caption{Heatmap of JRD score across different positions and angles for Dubins car. The black and gray circles trace out the boundaries for the slippery zone and obstacles respectively. The score remains low across the state space despite the presence of dynamics mismatch.}
    \label{fig:ood_heatmap}
\end{figure}

The bottom row shows the quantities used by different baselines for safety adaptation. AcoFi adapts the safety threshold online by computing a conformal score $s_t^{\text{AcoFi}} = \max(Q(z_t, a_t) - R_t, 0)$ with $Q$ being a state-action critic relating to the value function as $V(z_t) = \max_a Q(z_t, a)$, and $R_t = (1-\gamma)l(z_t) + \gamma\min(l(z_t), V(z_{t+1}))$ being a bootstrapped value function. A  quantile $q_t$ is then computed to adapt the safety threshold as $\delta_{t+1} \leftarrow q_{t+1} + \gamma\delta + (1-\gamma)l(z_{t+1})$. 
The bottom-left plot of Figure \ref{fig:dubins_study} displays $s_t^{\text{AcoFi}}$ for the rollout. 
Notably, the score fails to reflect the dynamics shift introduced when the car enters the slippery zone at time $t=10$. Instead, the score remains zero for several time steps, revealing that $Q(z_t, a_t) \leq R_t$, i.e the bootstrapped value estimate believes the trajectory is becoming \textit{safer} despite the increased velocity, resulting in no adaptation of the safety threshold (Figure \ref{fig:dubins_study}(b)) and the subsequent collision. 
This occurs because a critic trained on rollouts from an incorrect world model can itself become incorrect. The score $s_t^{\text{AcoFi}}$ therefore measures self-consistency between $Q$ and its bootstrapped target $R_t$, which need not reveal accumulated error in the world model.

UNISafe uses the JRD score of a probabilistic ensemble to detect OOD latent states, and trains the margin $l(z_t)$ to flag states above a calibrated threshold as unsafe. This score is shown in Figure \ref{fig:dubins_study}(e). We observe that despite the abrupt increase in velocity at $t=10$, the score does not undergo any significant change, remaining below the calibrated threshold until the moment of collision. In this case, the ensemble is \textit{confidently wrong}, producing similar predictions despite the shift in the underlying system dynamics. This is further demonstrated in Figure \ref{fig:ood_heatmap}, where we show the JRD score heatmap slices across different heading angles.
The score remains low and below the OOD threshold despite the dynamics shift induced by the slippery region. Importantly, there is no noticeable difference between the scores inside and outside the slippery region, indicating that OOD uncertainty does not always capture inaccuracies in the world model.
\begin{figure}[t]
    \centering
    \includegraphics[width=\linewidth]{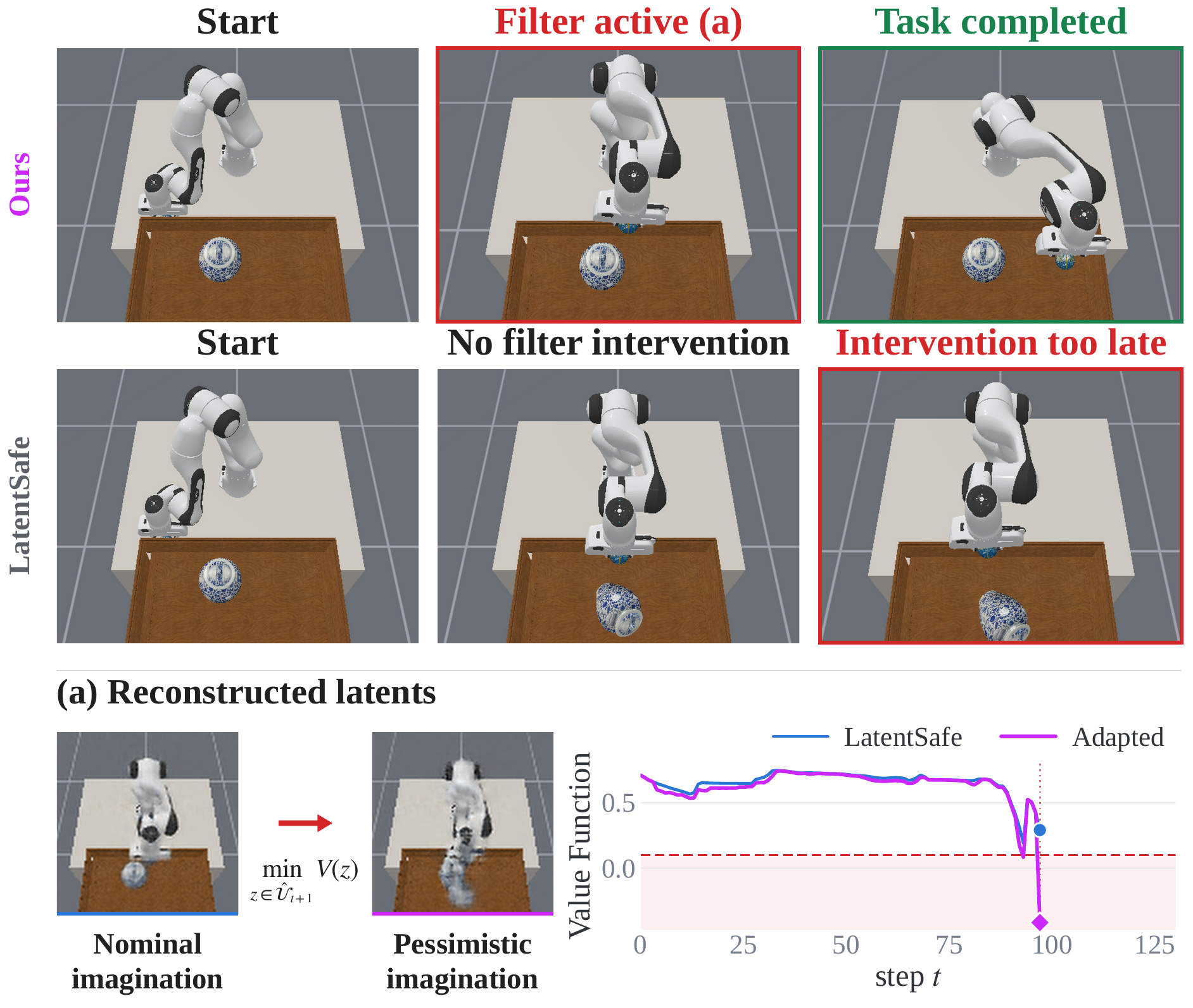}
    \caption{Our method imagines the worst-case outcome given the estimated uncertainty set, allowing it to adapt to errors in the world model.}
    \label{fig:franka_example}
\end{figure}
Our approach instead leverages the online latent prediction error $e_t$, shown in Figure \ref{fig:dubins_study}(d). Notably, this error rises sharply when the car enters the slippery zone, as the world model predictions no longer align with the observations. Algorithm \ref{alg:pess_safety_filering} expands the uncertainty set in response to this error, allowing the value function to quickly adapt (Figure \ref{fig:dubins_study}(b)). This triggers an early safety intervention which steers the car away from the obstacle.

This example highlights two benefits of adapting the value function with latent prediction error. First, the error is an observation-grounded signal of world-model inaccuracy, giving the filter a direct reactive mechanism at deployment. Second, UNISafe and AcoFi instead rely on auxiliary quantities trained separately from the deployed world model. If those are themselves miscalibrated, an untrustworthy model still looks trustworthy.

\subsection{Case Study 2: Franka Panda Arm}\label{franka_case_study}
We next study our method in a simulated manipulation task using a Franka Panda arm. The task is to pick up a globe and place it on a shelf without toppling a vase along the way. We implement this environment in ManiSkill \cite{taomaniskill3}, and train a world model and the safety margin function using a dataset of 3000 scripted trajectories, consisting of 1286 failed and 1714 safe examples. Each trajectory is initialized by randomizing the starting positions of the globe and vase. The observations consist of wrist-mounted and top-down camera views of the robot, along with proprioceptive inputs. Actions are 6-dimensional delta increments for position and orientation concatenated with a discrete gripper command.
The nominal policy is a diffusion policy \cite{chi2025diffusion} trained using a dataset of 600 simulated trajectories.

\begin{figure}[t!]
    \centering
    \includegraphics[width=\linewidth]{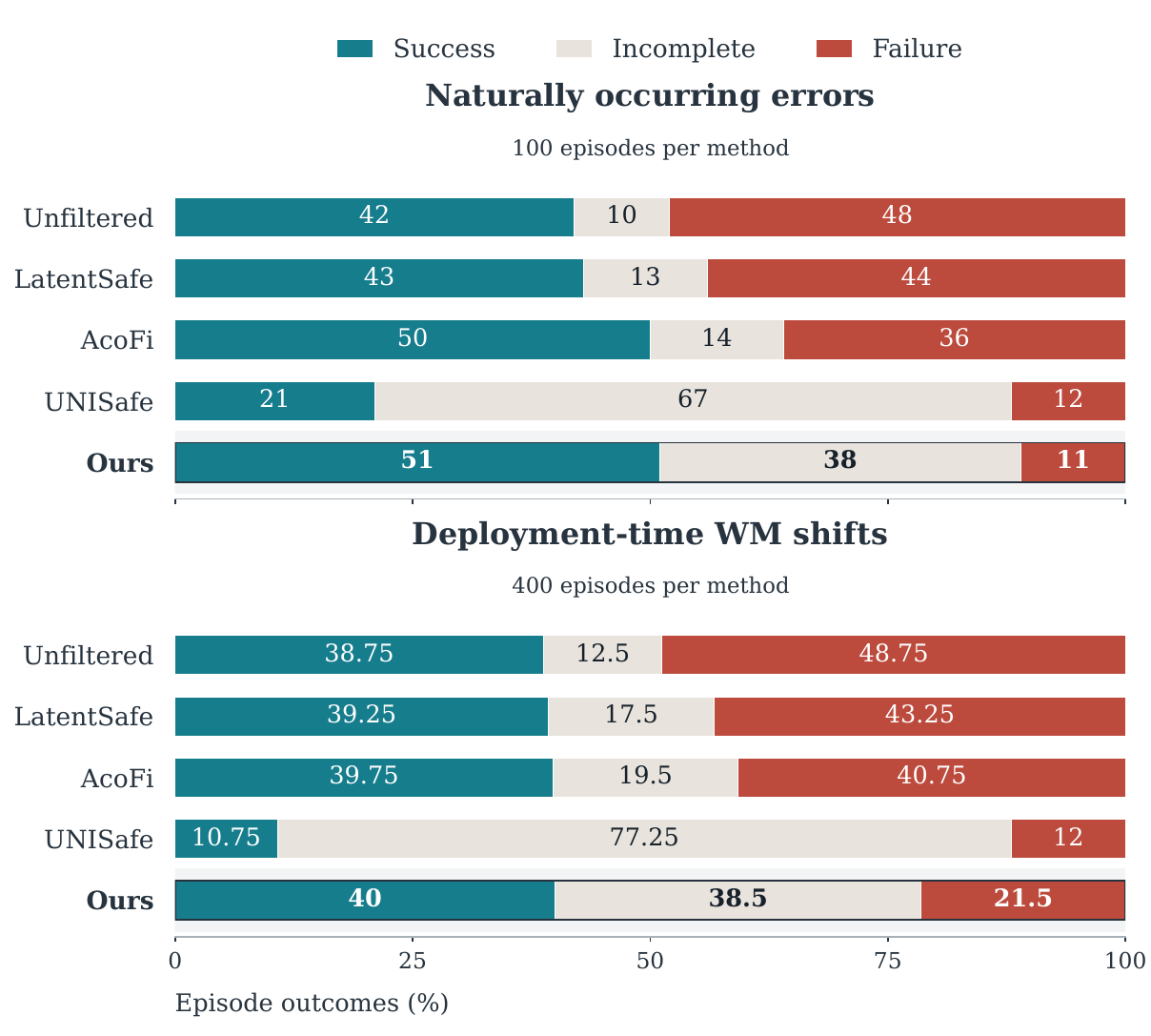}
    \caption{Franka Panda simulation results with a nominal diffusion policy. Our approach drastically reduces failures without suffering from overconservatism.}
    \label{fig:sim_results}
\end{figure}

\textbf{Qualitative Evaluation.} We begin by illustrating our method qualitatively through Figure \ref{fig:franka_example}, contrasted against LatentSafe, which does not have a mechanism for online adaptation. Due to natural errors in the world model, the nominal value function $V$ does not accurately capture the boundary of the unsafe set. As a result, LatentSafe triggers the safety intervention only after the vase has tilted beyond the point of no return. By continuously adapting the value function using online prediction error, our method remains robust to world model inaccuracies. This enables preemptive engagement of the safety fallback, allowing the arm to place the globe on the shelf without toppling the vase. 

By training a reconstructor alongside the world model, we can visualize both the nominal latent prediction and the worst-case latent state found by solving \eqref{opt_prob} at the time of filter intervention. These reconstructions are shown at the bottom of Figure~\ref{fig:franka_example}. In this example, Algorithm \ref{alg:pess_safety_filering} imagines that, under the current uncertainty estimate, the proposed action could lead to the vase toppling. This explains the discrepancy between the nominal and adapted value functions in the bottom-right  of Figure \ref{fig:franka_example}: our method assigns the pessimistic $\bar{V}$ to be much lower than the nominal $V$.

\textbf{Quantitative Evaluation.} Next, we quantitatively evaluate the performance of our method.
We use $q_\text{max}=2.0$, computed with $\alpha_\text{cal} = 0.05$ and $\alpha_\text{traj} = 0.1$ on a calibration set of size 100. For ACI, we choose $\alpha=0.2$ with step size $\lambda=0.05$. A threshold $\delta=0.15$ is used for safety filtering across all baselines, tuned to balance safety and performance. 
We run episodes under two types of world-model error:
\begin{itemize}
    \item \textit{Naturally occurring prediction errors:} 
    These are errors that arise from imperfect world-model learning under the nominal deployment distribution. This setting evaluates if the safety filter can remain effective when the world model is imperfect but not explicitly perturbed.
    \item \textit{Deployment-time world-model shifts:} We induce additional mismatch between the world model and deployment observations through two probes. First, we introduce a \textit{temporal abstraction mismatch} by changing the effective transition scale of the world model from its nominal $\Delta t=0.1$s to $\Delta t=0.2$s or $\Delta t=0.05$s. Second, we introduce \textit{observation staleness} by conditioning the world model on delayed observations $o_{t-L}$ in \eqref{world_model}, with $L \in \{0,2,4\}$. These probes alter the consistency between the model's latent predictions and the observations received during deployment, producing systematic latent prediction errors.
\end{itemize}

The results are given in Figure \ref{fig:sim_results}. Notably, our method achieves a comparable success rate while suffering the fewest failures in the in-distribution setting with natural world model errors. After introducing deployment-time world-model shifts, our approach still retains the highest success rate, while being significantly safer than LatentSafe and AcoFi. The rise in incompletion rates under our method reflects a division of labor, where the filter must intervene early enough for $\pi_\text{safe}$ to retract the system to a recoverable state, but whether the episode then completes depends on $\pi_\text{nom}$ resuming post-intervention, which it was not trained to do. The outcomes are also not equally costly, since failures are irreversible while incompletions leave the system intact and allow for continued goal-directed behavior.

\begin{figure}[!t]
    \centering
    \includegraphics[width=\linewidth]{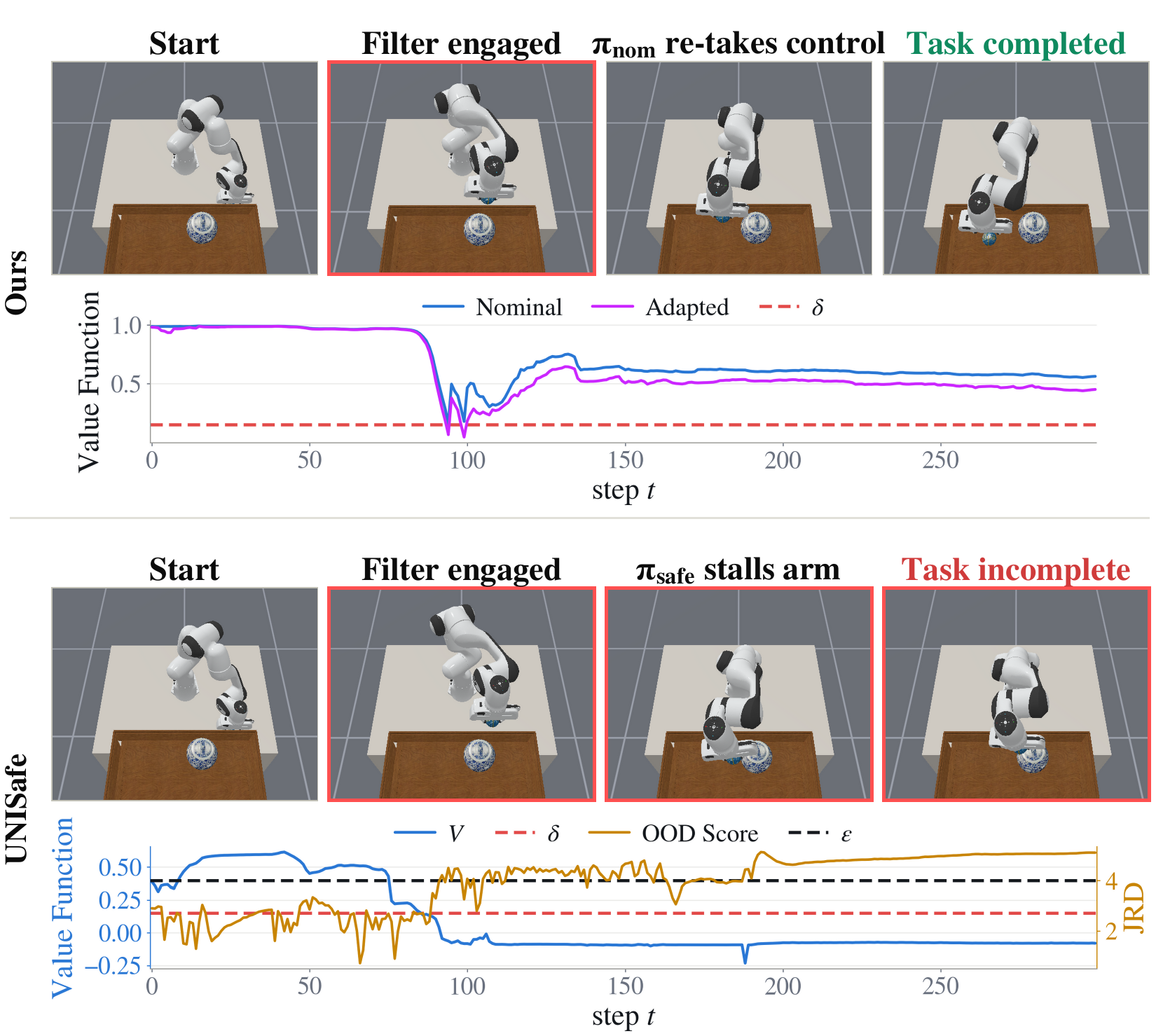}
    \caption{(Top) Our method adapts the value function using estimated model errors to avoid failure while retaining task performance. (Bottom) UNISafe enables safety at the cost of frequently driving the system into OOD states where it is unable to recover.}
    \label{fig:unisafe_compare}
\end{figure}

\begin{figure*}[!ht]
    \centering
    \includegraphics[width=0.85\linewidth]{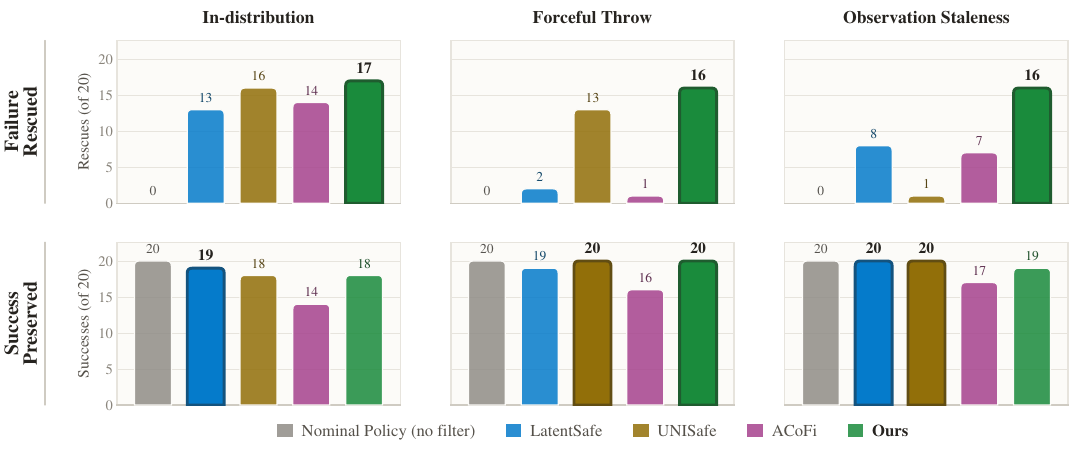}
    \caption{Hardware evaluations on Franka Panda across three different testing conditions. Top row shows the number of failed nominal trajectories rescued (higher is better). Bottom row shows the number of successful nominal trajectories still succeeding under each method (higher is better).}
    \label{fig:hardware_results}
\end{figure*}
\begin{figure}[!ht]
    \centering
    \includegraphics[width=0.9\linewidth]{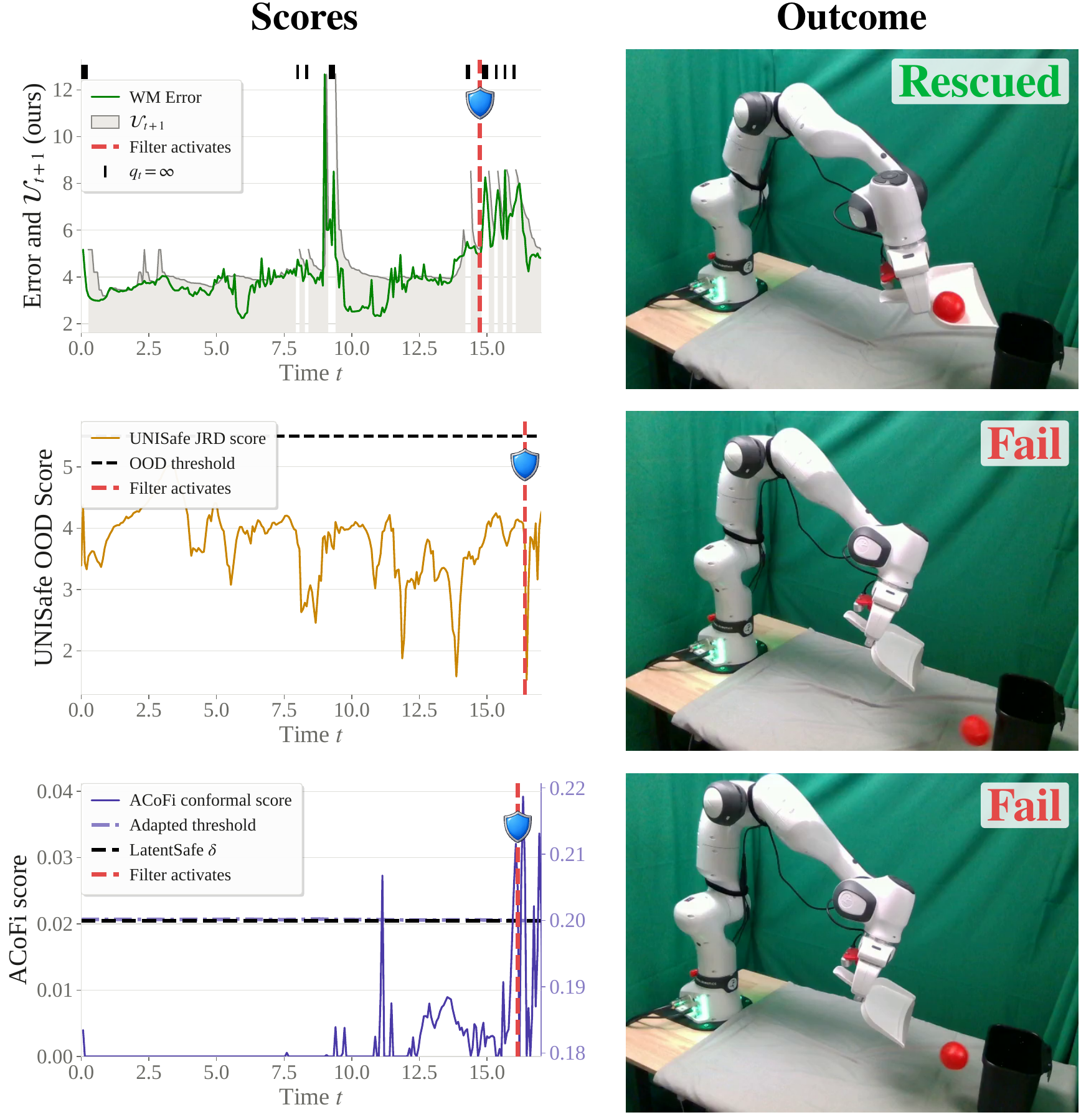}
    \caption{By adapting the safety filter using the estimated world model errors, our method is able to rescue a failed trajectory by engaging the safety policy earlier than baselines. \vspace{-0.5cm}}
    \label{fig:hardware_example}
\end{figure}

UNISafe attains the lowest failure rate during deployment time WM shifts, but does so at the cost of substantial conservatism resulting in frequent incompletions. 
This is visualized in Figure \ref{fig:unisafe_compare} (bottom).
UNISafe frequently classifies the latent state as OOD (yellow line repeatedly crossing the black threshold). We hypothesize that this occurs because the manipulation setting is substantially more complex than the Dubins car example, with higher-dimensional observations, a larger action space, and more diverse contact-rich transitions, making ensemble disagreement more likely during deployment. Since UNISafe purposefully treats OOD states as unsafe, this results in a highly conservative value-function landscape that repeatedly triggers the safety policy, ultimately causing the robot to stall.
 
In contrast, our method does not actively avoid OOD regions, but rather leverages observed evidence to estimate how wrong the world model is in those areas. This facilitates recovery by enabling the safety policy to retract the arm away from the vase, allowing the value function to rise and hand control back to the nominal policy (top row, Figure \ref{fig:unisafe_compare}). 
%

\section{Hardware Experiments}
We evaluate our method on hardware using a Franka Panda platform operating at 15 Hz, where the task is to use a dustpan to throw a tomato into a trash can. Failure is defined as the tomato landing outside the trash can. We evaluate our method by deploying it on both successful and failed trajectory replays. Each type of replay is divided into three categories, \textit{\textbf{1) In-distribution:}} the trajectory comes from the same distribution as the WM training dataset, \textit{\textbf{2) Forceful Throw:}} the tomato is suddenly jerked out of the dustpan, \textit{\textbf{3) Observation Staleness:}} The world model is conditioned on delayed observations $o_{t-L}$ in the same manner as in section \ref{franka_case_study}. We use $L=10$ for all observation staleness experiments. A collection of 20 replays is collected for both success and failure in each individual category.

We train a Dreamer world model and safety margin function using a dataset of 350 successful and 150 failed trajectories, collected using teleoperation. The observations consist of side and top-down views of the robot along with a wrist view and proprioception. We use $\delta=0.2$ across all experiments and $q_\text{max}=4.5$, derived from a calibration dataset of 30 trajectories with $\alpha_\text{cal} = \alpha_\text{traj} = 0.25$. ACI parameters are chosen as $\alpha=0.2$ and $\lambda=0.05$. Our method with PGD averages $39.59$ ms per step on an NVIDIA GeForce RTX 5090  (LatentSafe: $18.09$ ms, AcoFi: $17.17$ ms, UNISafe: $17.42$ ms), well within the 15 Hz control budget.
%

\textbf{Metrics.} Evaluation is based on the number of failed trajectories rescued, and the number of successful outcomes preserved.

Results from 480 runs are shown in Figure \ref{fig:hardware_results}. First, as expected, when the deployment setting remains within the world model's training distribution, all methods lead to a significant number of rescues among trajectories that would otherwise fail.
When the tomato is forcefully thrown, LatentSafe and ACoFi are largely unable to prevent failures, reflecting their limited ability to adapt to the sudden dynamics mismatch. In contrast, our method maintains a high rescue rate by adapting the value function pessimistically in response to the observed world-model errors, enabling preemptive engagement of the safety fallback before the tomato is thrown out of the pan. 
UNISafe likewise maintains a high number of rescues, as the forceful throw is detected as an OOD action that should be safeguarded against.
Under observation staleness, our method continues to enable a significant number of rescues, while performance degrades substantially across all baselines. The strongest baseline achieves $50\%$ fewer rescues than our method. 
We hypothesize that this is because stale observations induce incorrect latent predictions while still remaining within the world model's learned distribution, causing UNISafe to under-respond to the resulting model error.
Moreover, because these experiments use open-loop replays, UNISafe does not experience the persistent covariate shift or robot stalling observed in the closed-loop simulation study in Section~\ref{franka_case_study}.
 
This behavior is also evident in the qualitative comparison in Figure \ref{fig:hardware_example}, which shows one rollout from the observation-staleness setting evaluated under different methods.
The UNISafe JRD score remains below its OOD threshold and does not exhibit a strong response to the prediction mismatch, indicating agreement among the probabilistic ensemble despite elevated model error.
Similarly, the ACoFi nonconformity score remains low in magnitude throughout the episode and does not respond noticeably to model inaccuracies.
 
In contrast, our method directly tracks the discrepancy between the predicted and inferred latent states. Under stale observations, this error remains elevated throughout the episode and rises sharply near the end as the dustpan tilts downward, producing action inputs that are increasingly inconsistent with the stale observations (the mid-episode spike is caused by the chaotic bouncing dynamics of the tomato as the arm grasps the dustpan). This error-aware adaptation allows our method to lower the pessimistic value estimate early enough for the safety policy to prevent failure.

\section{Conclusions and Limitations}
We present a method for online adaptation of world-model based safety filters from prediction error estimates. By continuously estimating latent uncertainty sets during operation, we compute the worst-case imagination under the current estimated uncertainty. Through both qualitative and quantitative demonstrations across simulation and hardware, we show that our error-aware safety filter results in improved safety over existing methods.

Our work is not without limitations. While providing additional robustness against world model errors, our method is reactive and does not provide formal safety guarantees. Solving \eqref{opt_prob} incurs additional computational cost compared with existing methods. Moreover, Theorem \ref{thm_1} assumes exchangeability between deployment and calibration trajectories, which can be violated under safety filtering if recovery actions are absent from the calibration set. These constitute important avenues for future research.


\bibliographystyle{IEEEtran}
\bibliography{bib}

@article{seo2025uncertainty,
  title={Uncertainty-aware latent safety filters for avoiding out-of-distribution failures},
  author={Seo, Junwon and Nakamura, Kensuke and Bajcsy, Andrea},
  journal={arXiv preprint arXiv:2505.00779},
  year={2025}
}

@article{huriot2026safe,
  title={Safe Control using Learned Safety Filters and Adaptive Conformal Inference},
  author={Huriot, Sacha and Tabbara, Ihab and Sibai, Hussein},
  journal={arXiv preprint arXiv:2604.18482},
  year={2026}
}

@article{ren2023robots,
  title={Robots that ask for help: Uncertainty alignment for large language model planners},
  author={Ren, Allen Z and Dixit, Anushri and Bodrova, Alexandra and Singh, Sumeet and Tu, Stephen and Brown, Noah and Xu, Peng and Takayama, Leila and Xia, Fei and Varley, Jake and others},
  journal={arXiv preprint arXiv:2307.01928},
  year={2023}
}

@article{gibbs2021adaptive,
  title={Adaptive conformal inference under distribution shift},
  author={Gibbs, Isaac and Candes, Emmanuel},
  journal={Advances in Neural Information Processing Systems},
  volume={34},
  pages={1660--1672},
  year={2021}
}

@article{nakamura2025train,
  title={How to train your latent control barrier function: Smooth safety filtering under hard-to-model constraints},
  author={Nakamura, Kensuke and Bishop, Arun L and Man, Steven and Johnson, Aaron M and Manchester, Zachary and Bajcsy, Andrea},
  journal={arXiv preprint arXiv:2511.18606},
  year={2025}
}

@article{taomaniskill3,
  title={ManiSkill3: GPU Parallelized Robotics Simulation and Rendering for Generalizable Embodied AI},
  author={Stone Tao and Fanbo Xiang and Arth Shukla and Yuzhe Qin and Xander Hinrichsen and Xiaodi Yuan and Chen Bao and Xinsong Lin and Yulin Liu and Tse-kai Chan and Yuan Gao and Xuanlin Li and Tongzhou Mu and Nan Xiao and Arnav Gurha and Viswesh Nagaswamy Rajesh and Yong Woo Choi and Yen-Ru Chen and Zhiao Huang and Roberto Calandra and Rui Chen and Shan Luo and Hao Su},
  journal = {Robotics: Science and Systems},
  year={2025},
}

@article{chi2025diffusion,
  title={Diffusion policy: Visuomotor policy learning via action diffusion},
  author={Chi, Cheng and Xu, Zhenjia and Feng, Siyuan and Cousineau, Eric and Du, Yilun and Burchfiel, Benjamin and Tedrake, Russ and Song, Shuran},
  journal={The International Journal of Robotics Research},
  volume={44},
  number={10-11},
  pages={1684--1704},
  year={2025},
  publisher={Sage Publications Sage UK: London, England}
}

@inproceedings{kim2025cosmospolicy,
  title={Cosmos Policy: Fine-Tuning Video Models for Visuomotor Control and Planning},
  author={Kim, Moo Jin and Gao, Yihuai and Lin, Tsung-Yi and Lin, Yen-Chen and Ge, Yunhao and Lam, Grace and Liang, Percy and Song, Shuran and Liu, Ming-Yu and Finn, Chelsea and Gu, Jinwei},
  booktitle={International Conference on Learning Representations (ICLR)},
  year={2026}
}

@article{intelligence2025pi_,
  title={{$\pi_{0.5}$}: a {Vision-Language-Action} Model with {Open-World} Generalization},
  author={Intelligence, Physical and Black, Kevin and Brown, Noah and Darpinian, James and Dhabalia, Karan and Driess, Danny and Esmail, Adnan and Equi, Michael and Finn, Chelsea and Fusai, Niccolo and others},
  journal={arXiv preprint arXiv:2504.16054},
  year={2025}
}

@article{team2025gemini,
  title={Gemini robotics: Bringing ai into the physical world},
  author={Team, Gemini Robotics and Abeyruwan, Saminda and Ainslie, Joshua and Alayrac, Jean-Baptiste and Arenas, Montserrat Gonzalez and Armstrong, Travis and Balakrishna, Ashwin and Baruch, Robert and Bauza, Maria and Blokzijl, Michiel and others},
  journal={arXiv preprint arXiv:2503.20020},
  year={2025}
}

@article{nakamura2025generalizing,
  title={Generalizing safety beyond collision-avoidance via latent-space reachability analysis},
  author={Nakamura, Kensuke and Peters, Lasse and Bajcsy, Andrea},
  journal={arXiv preprint arXiv:2502.00935},
  year={2025}
}

@inproceedings{hafner2019learning,
  title={Learning latent dynamics for planning from pixels},
  author={Hafner, Danijar and Lillicrap, Timothy and Fischer, Ian and Villegas, Ruben and Ha, David and Lee, Honglak and Davidson, James},
  booktitle={International conference on machine learning},
  pages={2555--2565},
  year={2019},
  organization={PMLR}
}

@inproceedings{ames2019control,
  title={Control barrier functions: Theory and applications},
  author={Ames, Aaron D and Coogan, Samuel and Egerstedt, Magnus and Notomista, Gennaro and Sreenath, Koushil and Tabuada, Paulo},
  booktitle={2019 18th European control conference (ECC)},
  pages={3420--3431},
  year={2019},
  organization={Ieee}
}

@inproceedings{bansal2017hamilton,
  title={Hamilton-jacobi reachability: A brief overview and recent advances},
  author={Bansal, Somil and Chen, Mo and Herbert, Sylvia and Tomlin, Claire J},
  booktitle={2017 IEEE 56th annual conference on decision and control (CDC)},
  pages={2242--2253},
  year={2017},
  organization={IEEE}
}

@inproceedings{abbeel2006using,
  title={Using inaccurate models in reinforcement learning},
  author={Abbeel, Pieter and Quigley, Morgan and Ng, Andrew Y},
  booktitle={Proceedings of the 23rd international conference on Machine learning},
  pages={1--8},
  year={2006}
}

@article{berger2026biased,
  title={Biased Dreams: Limitations to Epistemic Uncertainty Quantification in Latent Space Models},
  author={Berger, Julia and Frauenknecht, Bernd and Trimpe, Sebastian and Leibe, Bastian},
  journal={arXiv preprint arXiv:2604.25416},
  year={2026}
}

@article{ward2026foundational,
  title={Foundational world models accurately detect bimanual manipulator failures},
  author={Ward, Isaac R and Ho, Michelle and Liu, Houjun and Feldman, Aaron and Vincent, Joseph and Kruse, Liam and Cheong, Sean and Eddy, Duncan and Kochenderfer, Mykel J and Schwager, Mac},
  journal={arXiv preprint arXiv:2603.06987},
  year={2026}
}

@inproceedings{fisac2019bridging,
  title={Bridging hamilton-jacobi safety analysis and reinforcement learning},
  author={Fisac, Jaime F and Lugovoy, Neil F and Rubies-Royo, Vicen{\c{c}} and Ghosh, Shromona and Tomlin, Claire J},
  booktitle={2019 International Conference on Robotics and Automation (ICRA)},
  pages={8550--8556},
  year={2019},
  organization={IEEE}
}

@article{hafner2023mastering,
  title={Mastering diverse domains through world models},
  author={Hafner, Danijar and Pasukonis, Jurgis and Ba, Jimmy and Lillicrap, Timothy},
  journal={arXiv preprint arXiv:2301.04104},
  year={2023}
}

@article{an2026feedback,
  title={Feedback world model enables precise guidance of diffusion policy},
  author={An, Tuo and Jia, Jindou and Li, Gen and Li, Jingliang and Zhou, Chuhao and Liu, Pengfei and Lyu, Bofan and Bai, Jiaqi and Guo, Xinying and Li, Geng and others},
  journal={arXiv preprint arXiv:2605.15705},
  year={2026}
}

@article{hansen2026hallucination,
  title={Hallucination in World Models is Predictable and Preventable},
  author={Hansen, Nicklas and Wang, Xiaolong},
  journal={arXiv preprint arXiv:2606.27326},
  year={2026}
}

@article{li2025robotic,
  title={Robotic world model: A neural network simulator for robust policy optimization in robotics},
  author={Li, Chenhao and Krause, Andreas and Hutter, Marco},
  journal={arXiv preprint arXiv:2501.10100},
  year={2025}
}

\section{Appendix}
\textbf{Proof of Theorem \ref{thm_1}.} We begin by noting the equivalence of the events
\begin{align}\label{event_equiv}
    \{Q_{\tau}^{\alpha_\text{traj}} \leq q_\text{max}\} = \{\frac{1}{T}\sum_{t=1}^T\mathbf{1}[e_t > q_\text{max}] \leq \alpha_\text{traj}\}.
\end{align}
From Proposition \ref{prop:aci_coverage}, we have that with probability $1$
\begin{align}\label{aci_error_guarantee}
    \frac{1}{T}\sum_{t=1}^T \mathbf{1}[e_t >q_t] - \alpha \leq r_T.
\end{align}
Combining \eqref{event_equiv} and \eqref{aci_error_guarantee} gives
\begin{align}\label{proof_step}
    &\mathbb{P}\left( \frac{1}{T}\sum_{t=1}^T \mathbf{1}[e_t > q_t] + \mathbf{1}[e_t >q_\text{max}] \leq \alpha + \alpha_\text{traj} + r_T\right) \geq \notag \\
    &\geq 1 - \alpha_\text{cal}.
\end{align}
Lastly, the inequality $\mathbf{1}[e_t > \bar{q}_t] \leq \mathbf{1}[e_t > q_t] + \mathbf{1}[e_t > q_\text{max}]$ allows us to upper-bound the left-hand side of \eqref{proof_step} with
\begin{align}
    \mathbb{P}\left( \frac{1}{T}\sum_{t=1}^T \mathbf{1}[e_t > \bar{q}_t] \leq \alpha + \alpha_\text{traj} + r_T\right),
\end{align}
giving us \eqref{thm1_ineq}. \hfill $\blacksquare$

\end{document}